\documentclass[accepted]{safeai2026} 

\usepackage[american]{babel}
\usepackage{natbib}
\usepackage{mathtools}
\usepackage{amssymb}
\usepackage{amsthm}
\usepackage{booktabs}
\usepackage{graphicx}

\theoremstyle{plain}
\newtheorem{theorem}{Theorem}
\newtheorem{proposition}{Proposition}

\theoremstyle{definition}
\newtheorem{definition}{Definition}
\theoremstyle{remark}

\newcommand{\R}{\mathbb{R}}
\newcommand{\E}{\mathbb{E}}
\newcommand{\Sset}{\mathcal{S}}

\newcommand{\Zset}{\mathcal{Z}}
\newcommand{\Uset}{\mathcal{U}}
\newcommand{\simplex}{\Delta}
\newcommand{\Wass}{\mathcal{W}_1}
\newcommand{\KL}{\mathrm{KL}}
\newcommand{\Ent}{\mathcal{H}}
\newcommand{\EVaR}{\mathrm{EVaR}}
\newcommand{\WEVaR}{\mathrm{WEVaR}}
\newcommand{\Lip}{\mathrm{Lip}}
\newcommand{\esssup}{\operatorname*{ess\,sup}}
\newcommand{\Pbar}{\bar{P}}

\title{Robust Risk Under Evolving Uncertainty: A Wasserstein Counterpart of the Entropic Value-at-Risk}

\author[1]{\href{mailto:diganto.ganguly@gmail.com}{Deep~Ganguly}\thanks{Funded by the Deutsche Forschungsgemeinschaft (DFG) research training group ConVeY (GRK~2428).}}
\author[1,2]{Jan~K\v{r}et\'insk\'y}
\affil[1]{Technical University of Munich, Germany}
\affil[2]{Masaryk University, Brno, Czech Republic}

\begin{document}
\maketitle

\begin{abstract}
An agent still learning its environment should be cautious while ignorant and bold once confident. The entropic value-at-risk captures this through a robust-optimization identity---a confidence level fixes the radius of a relative-entropy ball of alternative models---but that ball cannot reach catastrophes the nominal deems impossible, precisely what a safe agent must hedge. We instead use an optimal-transport ball and study the coherent risk measure it induces, the Wasserstein entropic value-at-risk. It has a variational dual mirroring the entropic formula (an inverse temperature becomes a transport price), occupies a definite place in the risk hierarchy, and provably accounts for the reachable catastrophes the entropic measure ignores; we verify both dualities numerically. Driving the transport radius by belief entropy then yields a closed-form robust dynamic-programming operator whose caution contracts as the belief sharpens, with a certified safety sandwich and a sharp safety switch.
\end{abstract}

\section{Risk Under Evolving Uncertainty}\label{sec:intro}
Consider a drone crossing a canyon under uncertain wind. A \emph{high-cruise} policy is fast when the air is calm but is thrown into the cliffs by a sudden gust; a \emph{hover-and-creep} policy is safe in any wind but slow. The wind regime---a hidden type $z$ drawn once---is never observed directly; the drone sees only noisy signals and must infer $z$ as it flies. Two reflexes each fail. A controller that assumes the \emph{worst case forever} (robust control \citep{iyengar2005,nilim2005,wiesemann2013}) keeps hovering even after the air has visibly calmed: safe, but so conservative that operators switch it off. A controller that \emph{commits before its belief sharpens} (expected-value or Bayesian planning \citep{kaelbling1998,ghavamzadeh2015}) cruises on a hunch, and a single wrong guess about $z$ is fatal---an expected return that blends a benign mode with a catastrophic one is a number nobody actually attains. What is wanted is a controller whose caution \emph{evolves with its evidence}---maximal under ignorance, vanishing once the environment is identified---together with a safety bound that holds at \emph{every} stage of learning. The agent's risk attitude must therefore be a function of its \emph{evolving} epistemic state, and the entropy of its belief is a live readout of how much it should distrust its own model.

The entropic value-at-risk \citep{ahmadijavid2012} turns this intuition into algebra. It is the unique single-parameter coherent family that sweeps from the mean to the essential supremum, and it has an exact \emph{distributionally robust} representation: its confidence level $\alpha$ equals the radius $-\ln\alpha$ of a relative-entropy ball of alternative laws. \citet{ganguly2025evar} further show the induced optimization is convex and admits well-posed, convergent estimation. This robust reading is what we want for evolving uncertainty: ``how conservative should I be?'' becomes ``how large should my ambiguity set be?'', and the latter is answered by the agent's own uncertainty.

But the relative-entropy geometry has a blind spot that is decisive for \emph{safety}. A ball $\{Q:\KL(Q\Vert P)\le\rho\}$ contains only laws absolutely continuous with respect to the nominal $P$: if $P$ assigns zero probability to a catastrophic transition, no member assigns it positive probability, since $\KL$ is infinite off the support of $P$. As the agent learns and its nominal kernel concentrates away from rare disasters, the relative-entropy adversary loses the ability even to \emph{represent} those disasters, so the safety check becomes \emph{vacuous} precisely when overconfidence is most dangerous. The optimal-transport (Wasserstein) geometry has no such blind spot: it prices a perturbation by how far mass must physically move, so a verifier can still ask ``what if it is a storm?'' at a cost proportional to the distance to that outcome---\emph{reachable but expensive} \citep{esfahani2018,gao2023,blanchet2019}.

\paragraph{Related work.}
Risk-averse dynamic programming \citep{ruszczynski2010}, risk-sensitive control and reinforcement learning \citep{howard1972,chow2015,hau2023}, percentile and parameter-uncertainty MDPs \citep{delage2010}, constrained and safe RL \citep{altman1999,garcia2015,sui2015,brunke2022}, and adversarially- or Wasserstein-robust RL \citep{pinto2017,abdullah2019} all add caution to sequential decisions---but with a \emph{fixed} attitude. Our contribution is to let the ambiguity radius, hence the risk attitude, be read off the agent's evolving belief, and to give it the optimal-transport geometry safety requires \citep{kuhn2019}.

\paragraph{Contributions.}
We keep the entropic value-at-risk's robust formulation and its guarantees, and replace its ball.
\begin{itemize}[leftmargin=1.2em,itemsep=1pt,topsep=1pt]
\item We define the \textbf{Wasserstein entropic value-at-risk} ($\WEVaR$) and prove a Kantorovich--Rubinstein \textbf{variational dual} (Thm.~\ref{thm:dual}) that mirrors the entropic formula term for term, is \textbf{convex and well-posed} (the transport analogue of the guarantee of \citealp{ganguly2025evar}), and has a closed ``mean-plus-Lipschitz'' form.
\item We show it is \textbf{coherent} and place it in the risk hierarchy (Thm.~\ref{thm:hier}): it is sandwiched against the entropic measure by a transport--entropy inequality and \textbf{strictly accounts} for zero-nominal-probability catastrophes the entropic measure ignores.
\item We \textbf{verify} both robust-optimization problems, their dualities, and the comparison numerically with Gurobi (\S\ref{sec:verify}), confirming strong duality to solver tolerance.
\item Driving the transport radius by \textbf{belief entropy} gives a closed-form robust dynamic-programming operator (Thm.~\ref{thm:bellman}) and a computable safety switch under evolving belief (\S\ref{sec:evolving}).
\end{itemize}

\section{The Entropic Value-at-Risk and Its Robust Formulation}\label{sec:evar}
Let $X$ be a bounded loss on a finite space $\Sset$ with reference law $P$, and let $d:\Sset\times\Sset\to\R_{\ge0}$ be a ground metric with diameter $D=\max_{s,s'}d(s,s')$. A risk measure is \emph{coherent} \citep{artzner1999} if it is monotone, translation-invariant, positively homogeneous and subadditive; the conditional value-at-risk \citep{rockafellar2000} and the entropic value-at-risk are its canonical instances. The latter has the dual/primal pair
\begin{align}
\EVaR_\alpha(X)&=\inf_{t>0}\ \tfrac1t\big(\ln\E_P[e^{tX}]-\ln\alpha\big)\label{eq:evar-primal}\\
&=\sup\big\{\E_Q[X]:\KL(Q\Vert P)\le-\ln\alpha\big\},\label{eq:evar-dual}
\end{align}
where the inverse temperature $t$ in \eqref{eq:evar-primal} is the multiplier on the relative-entropy constraint in \eqref{eq:evar-dual}, and the worst-case law is the exponential tilt $Q^\star_s\propto P_s e^{t^\star X_s}$ \citep{ahmadijavid2012,ganguly2025evar}. Two facts we carry forward from \citet{ganguly2025evar}: \emph{(G1)} the reparametrized objective is convex, so \eqref{eq:evar-primal} is a well-posed one-dimensional convex program with a unique solution; \emph{(G2)} the tilt $Q^\star$ is supported on $\mathrm{supp}(P)$. Property (G2) is the blind spot: $\EVaR_\alpha(X)$ is independent of the value of $X$ on any zero-nominal-probability state, \emph{for every} $\alpha$.

\section{Swapping the Ball: A Wasserstein Robust Risk}\label{sec:wevar}
We keep the robust template \eqref{eq:evar-dual} and replace the relative-entropy ball by a $1$-Wasserstein ball, $\Wass(Q,P)=\min_{\gamma\in\Gamma(Q,P)}\sum_{s,s'}d(s,s')\gamma(s,s')$.

\begin{definition}[$\WEVaR$]\label{def:wevar}
For radius $\varepsilon\ge0$,
\begin{equation}\label{eq:wevar}
\WEVaR_\varepsilon(X)=\sup\big\{\E_Q[X]:\Wass(Q,P)\le\varepsilon\big\}.
\end{equation}
\end{definition}

\begin{theorem}[Variational dual, convexity, closed form]\label{thm:dual}
For bounded $X$ and $\varepsilon\ge0$,
\begin{equation}\label{eq:wevar-dual}
\WEVaR_\varepsilon(X)=\inf_{\lambda\ge0}\big\{\lambda\varepsilon+\E_P[\,X^{c}_\lambda\,]\big\},
\end{equation}
where $X^{c}_\lambda(s)=\max_{s'}(X(s')-\lambda\,d(s',s))$ is the $c$-transform of $X$. The objective is convex in $\lambda$, so \eqref{eq:wevar-dual} is a well-posed one-dimensional convex program (mirroring \emph{(G1)}); the optimal price satisfies $\lambda^\star\le\Lip_d(X)$; $\varepsilon\mapsto\WEVaR_\varepsilon(X)$ is concave and nondecreasing; and
\begin{equation}\label{eq:wevar-closed}
\WEVaR_\varepsilon(X)=\E_P[X]+\varepsilon\,\Lip_d(X)
\end{equation}
for $\varepsilon$ below a saturation threshold (in particular, exactly for two-point supports), where $\Lip_d(X)=\max_{s\ne s'}|X(s)-X(s')|/d(s,s')$.
\end{theorem}
\begin{proof}[Proof sketch]
Strong duality for Wasserstein DRO \citep{gao2023,blanchet2019} gives \eqref{eq:wevar-dual}; the $c$-transform $X^c_\lambda$ is the transport counterpart of the cumulant $\tfrac1t\ln\E_P e^{tX}$ in \eqref{eq:evar-primal}. Convexity in $\lambda$ holds because $X^c_\lambda$ is a pointwise maximum of affine functions of $\lambda$, hence convex, and $\E_P$ and the $\lambda\varepsilon$ term preserve convexity. For $\lambda\ge\Lip_d(X)$ the maximum is attained at $s'=s$, so the bracket equals $\E_P[X]$ and the objective is $\E_P[X]+\lambda\varepsilon$, minimized at $\lambda=\Lip_d(X)$; this gives \eqref{eq:wevar-closed} until transporting all displaceable mass to the maximizer of $X$ saturates the bound, after which the value rises concavely toward $\max_s X(s)$.
\end{proof}

Equations \eqref{eq:evar-primal} and \eqref{eq:wevar-dual} are the same template---$\inf$ over a one-dimensional dual variable of a \emph{smoothed} expectation plus radius-times-multiplier---with an entropic smoothing for the relative-entropy ball and a transport (inf-convolution) smoothing for the Wasserstein ball. The temperature $t$ and the transport price $\lambda$ are the same Lagrangian object on two ambiguity geometries.

\begin{proposition}[Coherence]\label{prop:coherent}
$\WEVaR_\varepsilon$ is a coherent risk measure for every $\varepsilon\ge0$.
\end{proposition}
\begin{proof}[Proof sketch]
It is the support function of the convex, compact set $\{Q:\Wass(Q,P)\le\varepsilon\}$, hence positively homogeneous and subadditive; $P$ lies in the set, giving monotonicity; and adding a constant to $X$ shifts every $\E_Q[X]$ equally, giving translation invariance.
\end{proof}

\begin{theorem}[Hierarchy and relation to the entropic measure]\label{thm:hier}
On a finite metric space with diameter $D$:
\textbf{(i)} $\WEVaR_0(X)=\E_P[X]$ and $\WEVaR_\varepsilon(X)\uparrow\esssup(X)$ as $\varepsilon\uparrow D$, so $\WEVaR$ sweeps the full hierarchy;
\textbf{(ii)} (sandwich) $\EVaR_\alpha(X)\le\WEVaR_{\,D\sqrt{-\frac12\ln\alpha}}(X)$ for all $\alpha\in(0,1]$;
\textbf{(iii)} (catastrophe) if $P(s^\star)=0$ and $X(s^\star)>\E_P[X]$, then $\EVaR_\alpha(X)$ is independent of $X(s^\star)$ for every $\alpha$, whereas $\WEVaR_\varepsilon(X)$ strictly increases in $X(s^\star)$ once $\varepsilon>\mathrm{dist}_d(s^\star,\mathrm{supp}\,P)$.
\end{theorem}
\begin{proof}[Proof sketch]
(i) is immediate from finiteness. (ii): Pinsker gives total variation $\le\sqrt{\KL/2}$ and $\Wass\le D\cdot\mathrm{TV}$, so the relative-entropy ball of radius $-\ln\alpha$ is contained in the Wasserstein ball of radius $D\sqrt{-\tfrac12\ln\alpha}$; take suprema \citep{bobkov1999,boucheron2013}. (iii) is \emph{(G2)} versus the transport reach: no finite-radius relative-entropy ball contains a Wasserstein ball, and the gap is exactly the reachable-but-zero-probability catastrophes.
\end{proof}

\begin{figure}[t]
  \centering
  \includegraphics[width=0.94\linewidth]{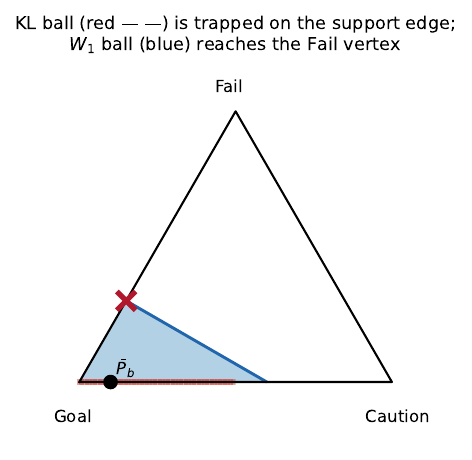}
  \caption{Why Wasserstein, geometrically. On the next-state simplex at a confident belief---the nominal $\Pbar_b$ places zero mass on \textsc{Fail}---the relative-entropy ball (red) is trapped on the support edge, blind to the catastrophe, while the $1$-Wasserstein ball (blue) reaches the \textsc{Fail} vertex; the worst-case adversary ($\times$) sits on its boundary. This is Thm.~\ref{thm:hier}(iii) drawn on $\simplex(\Sset)$.}\label{fig:klw1}
\end{figure}

\section{Empirical Guarantees}\label{sec:verify}
We verify both robust programs and their duals with Gurobi~12 on a five-state metric space ($X=[0,2,4,8,100]$, $d(s,s')=|s-s'|$). For the entropic measure we solve the relative-entropy-constrained program \eqref{eq:evar-dual} (a nonlinear program) and match it against the convex primal \eqref{eq:evar-primal}; for $\WEVaR$ we solve the transport linear program \eqref{eq:wevar} and match it against the Kantorovich--Rubinstein dual \eqref{eq:wevar-dual}.

Strong duality holds to solver tolerance: the entropic primal \eqref{eq:evar-primal} and the relative-entropy program \eqref{eq:evar-dual} agree to $1.2\times10^{-4}$, and the transport linear program agrees with its Kantorovich--Rubinstein dual to $9\times10^{-7}$. The closed form \eqref{eq:wevar-closed} is exact in the unsaturated regime (here $\varepsilon\le0.1$) and is otherwise upper-bounded by the exact one-dimensional dual, which the solver confirms. Across a sweep of radii both measures rise from the mean toward the worst case, with $\EVaR_\alpha\le\WEVaR$ at the Pinsker-matched radius (Thm.~\ref{thm:hier}(ii)). The catastrophe blind spot is decisive: with a zero-nominal-probability disaster whose loss we grow from $50$ to $1000$, the entropic value-at-risk changes by exactly $0$ at fixed confidence, while $\WEVaR$ at a fixed radius changes by $807.5$---the gap of Thm.~\ref{thm:hier}(iii), shown geometrically in Fig.~\ref{fig:klw1} and quantitatively across radii in Fig.~\ref{fig:dro} (App.~E).

\section{Evolving Uncertainty: Belief Entropy as the Radius}\label{sec:evolving}
We now let the radius \emph{evolve}. An environment has a hidden type $z\in\Zset$ drawn once; the agent maintains a Bayesian belief $b\in\simplex(\Zset)$ with update $\psi$, forms the belief-weighted nominal kernel $\Pbar_b(\cdot\mid s,a)=\sum_z b(z)P(\cdot\mid s,a,z)$, and sets the transport radius to the belief entropy, $\varepsilon(b)=\beta\,\Ent(b)$ with $\Ent(b)=-\sum_z b(z)\ln b(z)$ and sensitivity $\beta>0$, giving the ambiguity set $\Uset(b)=\{Q:\Wass(Q,\Pbar_b)\le\beta\Ent(b)\}$. As observations sharpen $b$, the entropy $\Ent(b)\downarrow0$, the ball contracts and drifts toward the nominal (Fig.~\ref{fig:balls}), and via Thm.~\ref{thm:dual} the agent's risk attitude slides from worst-case to risk-neutral---a dial driven by epistemic state. Since the controller minimizes value, the relevant functional is the infimal twin of Def.~\ref{def:wevar}, whose closed form is $\E_{\Pbar_b}[V]-\beta\Ent(b)\,\Lip_d(V)$.

\begin{figure}[t]
  \centering
  \includegraphics[width=\linewidth]{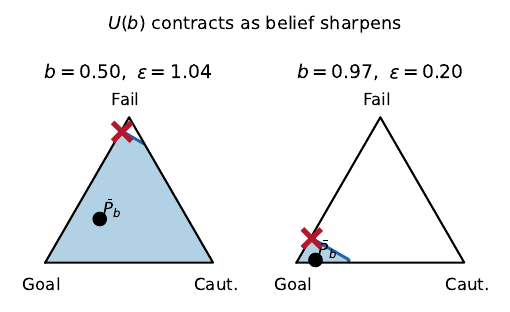}
  \caption{The entropy-modulated ambiguity set $\Uset(b)$ on the next-state simplex. As the belief sharpens the radius $\varepsilon=\beta\Ent(b)$ contracts ($1.04\!\to\!0.20$) and the nominal $\Pbar_b$ drifts toward the goal; the worst-case adversary ($\times$) loses its reach toward the failure vertex---robustness that vanishes with uncertainty (the geometric content of Thm.~\ref{thm:bellman}).}\label{fig:balls}
\end{figure}

\begin{theorem}[Closed-form robust update; contraction; safety]\label{thm:bellman}
For the operator, with discount $\gamma\in(0,1]$ and belief update $b'=\psi(b,s,a,s')$,
\[(\mathfrak{T}V)(s,b)=\max_{a}\big[r(s,a)+\gamma\!\inf_{Q\in\Uset(b)}\!\E_{Q}\,V(s',b')\big]:\]
\textbf{(i)} the inner problem has the closed form of Thm.~\ref{thm:dual}---no coupling linear program---reducing the per-$(s,a,b)$ cost to $O(|\Sset|^2)$;
\textbf{(ii)} $\mathfrak{T}$ is a $\gamma$-contraction in $\|\cdot\|_\infty$ when $\gamma<1$, and a contraction on proper MDPs when $\gamma=1$; either way it has a unique fixed point $V^\star$ \citep{bertsekas1996};
\textbf{(iii)} $V^\star$ obeys a safety sandwich $V^\star_{\mathrm{wc}}(s,b)\le V^\star(s,b)\le\E_{z\sim b}[V^\star_{\mathrm{opt}}(s,z)]$ between the always-maximally-cautious value and the type-aware oracle, and as $\Ent(b)\to0$ the radius vanishes and $V^\star(s,b)\to V^\star_{\mathrm{opt}}(s,z^\star)$.
\end{theorem}

Since $\Ent(b)\le\ln|\Zset|$, the worst-case radius is $\beta\ln|\Zset|$; requiring one action to remain viable under maximal ignorance gives the calibration $\beta<D/\ln|\Zset|$, the transport analogue of choosing a confidence level.

\begin{figure}[!htb]
  \centering
  \includegraphics[width=\linewidth]{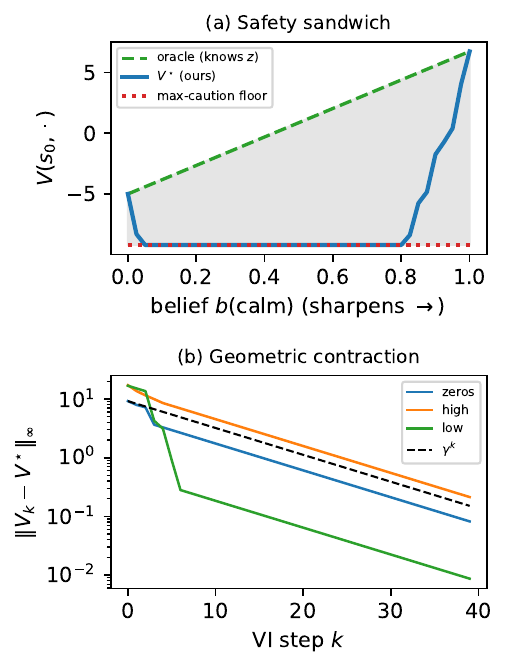}
  \caption{Computed guarantees on a five-state corridor with a hidden slip type ($\gamma=0.9$, Gurobi-checked). \textbf{(a)} The adaptive value $V^\star(s_0,b)$ stays inside the safety sandwich between the always-maximally-cautious floor $V_{\mathrm{wc}}$ and the type-aware oracle ceiling; the price of robustness shrinks as the belief sharpens. \textbf{(b)} From arbitrary initializations, value iteration with the closed-form operator collapses geometrically onto $V^\star$ along the $\gamma^k$ envelope, certifying the bound at every iterate.}\label{fig:sandwich}
\end{figure}

\paragraph{Evolving safety switch.}
On the Ambiguous Bridge ($\gamma=1$), \textsc{Sprint} reaches the goal ($+100$) under a benign type but falls to catastrophe ($-1000$) under an adversarial type, while \textsc{Crawl} is always safe but costly. With belief $b=[\alpha,1-\alpha]$ and $d(S_G,S_F)=1$, so $\Lip_d(V)=1100$, Thm.~\ref{thm:bellman}(i) gives in closed form $Q(\textsc{Sprint})=-1+1100(\alpha-\beta\Ent(\alpha))^{+}-1000$ and $Q(\textsc{Crawl})=80-1100\min(\beta\Ent(\alpha),1)$, which a Gurobi transport solve reproduces to $10^{-13}$. Equating them (the entropy terms cancel) yields a safety switch independent of $\beta$: the agent \textsc{Crawl}s until $\alpha\ge\alpha^\star=1081/1100\approx0.983$, then \textsc{Sprint}s. The threshold is set purely by the reward asymmetry. The same creep-then-cruise switch governs the canyon of \S\ref{sec:intro}: the drone hovers until its confidence that the air is calm crosses $\approx0.63$, then cruises (Fig.~\ref{fig:switch}, App.~E).

\paragraph{Guarantees, illustrated.}
On a five-state corridor whose forward action slips to a failure state only under a hidden ``storm'' type (Fig.~\ref{fig:sandwich}), the safety sandwich of Thm.~\ref{thm:bellman}(iii) holds at every belief: the adaptive value never drops below the always-maximally-cautious floor and never exceeds the type-aware oracle. The price of robustness is \emph{temporary}---the ceiling gap falls from $8.4$ at the uniform belief to $4.4$ once the belief reaches $b(\mathrm{calm})=0.95$, and to zero at identification. Value iteration with the closed-form operator converges from arbitrary initializations at empirical rate $0.90=\gamma$, and the closed-form inner update matches a Gurobi transport solve at every sampled belief (the underlying duality is verified to solver tolerance in \S\ref{sec:verify}), so a valid safety bound is available at every iterate, not only at convergence. The belief-simplex manifolds (Fig.~\ref{fig:manifold}) and the safety-vs-efficiency rollouts (Fig.~\ref{fig:pareto}) appear in App.~E.

\section{Discussion}\label{sec:disc}
The entropic and Wasserstein robust risks are the relative-entropy and optimal-transport faces of one idea---a coherent risk whose ambiguity radius is the agent's epistemic uncertainty---bridged by entropic optimal transport \citep{cuturi2013,peyre2019}; scaling the closed-form operator to continuous spaces via Lipschitz critics and to active information gathering are the natural next steps.

\bibliography{wevar}

\begin{thebibliography}{30}
\providecommand{\natexlab}[1]{#1}
\providecommand{\url}[1]{\texttt{#1}}
\expandafter\ifx\csname urlstyle\endcsname\relax
  \providecommand{\doi}[1]{doi: #1}\else
  \providecommand{\doi}{doi: \begingroup \urlstyle{rm}\Url}\fi

\bibitem[Abdullah et~al.(2019)Abdullah, Ren, Ammar, Milenkovic, Luo, Zhang, and
  Wang]{abdullah2019}
Mohammed~Amin Abdullah, Hang Ren, Haitham~Bou Ammar, Vladimir Milenkovic, Rui
  Luo, Mingtian Zhang, and Jun Wang.
\newblock Wasserstein robust reinforcement learning.
\newblock \emph{arXiv preprint arXiv:1907.13196}, 2019.

\bibitem[Ahmadi-Javid(2012)]{ahmadijavid2012}
Amir Ahmadi-Javid.
\newblock Entropic value-at-risk: A new coherent risk measure.
\newblock \emph{Journal of Optimization Theory and Applications}, 155\penalty0
  (3):\penalty0 1105--1123, 2012.

\bibitem[Altman(1999)]{altman1999}
Eitan Altman.
\newblock \emph{Constrained {Markov} Decision Processes}.
\newblock Chapman \& Hall/CRC, 1999.

\bibitem[Artzner et~al.(1999)Artzner, Delbaen, Eber, and Heath]{artzner1999}
Philippe Artzner, Freddy Delbaen, Jean-Marc Eber, and David Heath.
\newblock Coherent measures of risk.
\newblock \emph{Mathematical Finance}, 9\penalty0 (3):\penalty0 203--228, 1999.

\bibitem[Bertsekas and Tsitsiklis(1996)]{bertsekas1996}
Dimitri~P. Bertsekas and John~N. Tsitsiklis.
\newblock \emph{Neuro-Dynamic Programming}.
\newblock Athena Scientific, 1996.

\bibitem[Blanchet and Murthy(2019)]{blanchet2019}
Jose Blanchet and Karthyek Murthy.
\newblock Quantifying distributional model risk via optimal transport.
\newblock \emph{Mathematics of Operations Research}, 44\penalty0 (2):\penalty0
  565--600, 2019.

\bibitem[Bobkov and G{\"o}tze(1999)]{bobkov1999}
Sergey~G. Bobkov and Friedrich G{\"o}tze.
\newblock Exponential integrability and transportation cost related to
  logarithmic {Sobolev} inequalities.
\newblock \emph{Journal of Functional Analysis}, 163\penalty0 (1):\penalty0
  1--28, 1999.

\bibitem[Boucheron et~al.(2013)Boucheron, Lugosi, and Massart]{boucheron2013}
St{\'e}phane Boucheron, G{\'a}bor Lugosi, and Pascal Massart.
\newblock \emph{Concentration Inequalities: A Nonasymptotic Theory of
  Independence}.
\newblock Oxford University Press, 2013.

\bibitem[Brunke et~al.(2022)Brunke, Greeff, Hall, Yuan, Zhou, Panerati, and
  Schoellig]{brunke2022}
Lukas Brunke, Melissa Greeff, Adam~W. Hall, Zhaocong Yuan, Siqi Zhou, Jacopo
  Panerati, and Angela~P. Schoellig.
\newblock Safe learning in robotics: From learning-based control to safe
  reinforcement learning.
\newblock \emph{Annual Review of Control, Robotics, and Autonomous Systems},
  5:\penalty0 411--444, 2022.

\bibitem[Chow et~al.(2015)Chow, Tamar, Mannor, and Pavone]{chow2015}
Yinlam Chow, Aviv Tamar, Shie Mannor, and Marco Pavone.
\newblock Risk-sensitive and robust decision-making: A {CVaR} optimization
  approach.
\newblock In \emph{Advances in Neural Information Processing Systems
  (NeurIPS)}, 2015.

\bibitem[Cuturi(2013)]{cuturi2013}
Marco Cuturi.
\newblock Sinkhorn distances: Lightspeed computation of optimal transport.
\newblock In \emph{Advances in Neural Information Processing Systems
  (NeurIPS)}, 2013.

\bibitem[Delage and Mannor(2010)]{delage2010}
Erick Delage and Shie Mannor.
\newblock Percentile optimization for {Markov} decision processes with
  parameter uncertainty.
\newblock \emph{Operations Research}, 58\penalty0 (1):\penalty0 203--213, 2010.

\bibitem[Ganguly et~al.(2025)Ganguly, Girotra, Sekhar, and
  Joseph]{ganguly2025evar}
Deep Ganguly, Sarthak Girotra, Sirish Sekhar, and Ajin~George Joseph.
\newblock Risk-seeking reinforcement learning via multi-timescale entropic
  value-at-risk optimization.
\newblock \emph{Transactions on Machine Learning Research}, 2025.

\bibitem[Gao and Kleywegt(2023)]{gao2023}
Rui Gao and Anton~J. Kleywegt.
\newblock Distributionally robust stochastic optimization with {Wasserstein}
  distance.
\newblock \emph{Mathematics of Operations Research}, 48\penalty0 (2):\penalty0
  603--655, 2023.

\bibitem[Garc{\'i}a and Fern{\'a}ndez(2015)]{garcia2015}
Javier Garc{\'i}a and Fernando Fern{\'a}ndez.
\newblock A comprehensive survey on safe reinforcement learning.
\newblock \emph{Journal of Machine Learning Research}, 16:\penalty0 1437--1480,
  2015.

\bibitem[Ghavamzadeh et~al.(2015)Ghavamzadeh, Mannor, Pineau, and
  Tamar]{ghavamzadeh2015}
Mohammad Ghavamzadeh, Shie Mannor, Joelle Pineau, and Aviv Tamar.
\newblock Bayesian reinforcement learning: A survey.
\newblock \emph{Foundations and Trends in Machine Learning}, 8\penalty0
  (5--6):\penalty0 359--483, 2015.

\bibitem[Hau et~al.(2023)Hau, Petrik, and Ghavamzadeh]{hau2023}
Jia~Lin Hau, Marek Petrik, and Mohammad Ghavamzadeh.
\newblock Entropic risk optimization in discounted {MDP}s.
\newblock In \emph{International Conference on Artificial Intelligence and
  Statistics (AISTATS)}, 2023.

\bibitem[Howard and Matheson(1972)]{howard1972}
Ronald~A. Howard and James~E. Matheson.
\newblock Risk-sensitive {Markov} decision processes.
\newblock \emph{Management Science}, 18\penalty0 (7):\penalty0 356--369, 1972.

\bibitem[Iyengar(2005)]{iyengar2005}
Garud~N. Iyengar.
\newblock Robust dynamic programming.
\newblock \emph{Mathematics of Operations Research}, 30\penalty0 (2):\penalty0
  257--280, 2005.

\bibitem[Kaelbling et~al.(1998)Kaelbling, Littman, and
  Cassandra]{kaelbling1998}
Leslie~Pack Kaelbling, Michael~L. Littman, and Anthony~R. Cassandra.
\newblock Planning and acting in partially observable stochastic domains.
\newblock \emph{Artificial Intelligence}, 101\penalty0 (1--2):\penalty0
  99--134, 1998.

\bibitem[Kuhn et~al.(2019)Kuhn, Mohajerin~Esfahani, Nguyen, and
  Shafieezadeh-Abadeh]{kuhn2019}
Daniel Kuhn, Peyman Mohajerin~Esfahani, Viet~Anh Nguyen, and Soroosh
  Shafieezadeh-Abadeh.
\newblock {Wasserstein} distributionally robust optimization: Theory and
  applications in machine learning.
\newblock In \emph{Operations Research \& Management Science in the Age of
  Analytics (INFORMS TutORials)}, pages 130--166. INFORMS, 2019.

\bibitem[Mohajerin~Esfahani and Kuhn(2018)]{esfahani2018}
Peyman Mohajerin~Esfahani and Daniel Kuhn.
\newblock Data-driven distributionally robust optimization using the
  {Wasserstein} metric: Performance guarantees and tractable reformulations.
\newblock \emph{Mathematical Programming}, 171:\penalty0 115--166, 2018.

\bibitem[Nilim and El~Ghaoui(2005)]{nilim2005}
Arnab Nilim and Laurent El~Ghaoui.
\newblock Robust control of {Markov} decision processes with uncertain
  transition matrices.
\newblock \emph{Operations Research}, 53\penalty0 (5):\penalty0 780--798, 2005.

\bibitem[Peyr{\'e} and Cuturi(2019)]{peyre2019}
Gabriel Peyr{\'e} and Marco Cuturi.
\newblock Computational optimal transport.
\newblock \emph{Foundations and Trends in Machine Learning}, 11\penalty0
  (5--6):\penalty0 355--607, 2019.

\bibitem[Pinto et~al.(2017)Pinto, Davidson, Sukthankar, and Gupta]{pinto2017}
Lerrel Pinto, James Davidson, Rahul Sukthankar, and Abhinav Gupta.
\newblock Robust adversarial reinforcement learning.
\newblock In \emph{International Conference on Machine Learning (ICML)}, 2017.

\bibitem[Rockafellar and Uryasev(2000)]{rockafellar2000}
R.~Tyrrell Rockafellar and Stanislav Uryasev.
\newblock Optimization of conditional value-at-risk.
\newblock \emph{Journal of Risk}, 2:\penalty0 21--42, 2000.

\bibitem[Ruszczy{\'n}ski(2010)]{ruszczynski2010}
Andrzej Ruszczy{\'n}ski.
\newblock Risk-averse dynamic programming for {Markov} decision processes.
\newblock \emph{Mathematical Programming}, 125\penalty0 (2):\penalty0 235--261,
  2010.

\bibitem[Sui et~al.(2015)Sui, Gotovos, Burdick, and Krause]{sui2015}
Yanan Sui, Alkis Gotovos, Joel Burdick, and Andreas Krause.
\newblock Safe exploration for optimization with {Gaussian} processes.
\newblock In \emph{International Conference on Machine Learning (ICML)}, 2015.

\bibitem[Villani(2009)]{villani2009}
C{\'e}dric Villani.
\newblock \emph{Optimal Transport: Old and New}.
\newblock Springer, 2009.

\bibitem[Wiesemann et~al.(2013)Wiesemann, Kuhn, and Rustem]{wiesemann2013}
Wolfram Wiesemann, Daniel Kuhn, and Ber{\c{c}} Rustem.
\newblock Robust {Markov} decision processes.
\newblock \emph{Mathematics of Operations Research}, 38\penalty0 (1):\penalty0
  153--183, 2013.

\end{thebibliography}

\clearpage
\onecolumn
\appendix
\section*{Appendix: Extended Proofs}
Throughout, $\Sset$ is finite with $|\Sset|=n$; $d$ is a metric on $\Sset$ with diameter $D=\max_{s,s'}d(s,s')$; $P\in\simplex(\Sset)$ is the reference law; $X:\Sset\to\R$ is a bounded loss; and $\Wass(Q,P)=\min_{\gamma\in\Gamma(Q,P)}\sum_{s,s'}d(s,s')\,\gamma(s,s')$ with $\Gamma(Q,P)$ the couplings of $Q$ and $P$. We write $\E_Q[X]=\sum_s Q(s)X(s)$.

\subsection*{A.\;\;Theorem~\ref{thm:dual}}

\paragraph{(a) Duality.}
Since $\Wass(Q,P)$ is itself a minimum over couplings, the worst-case expectation is the value of a single linear program over the coupling $\gamma\ge0$ whose first marginal is fixed to $P$:
\begin{equation}\label{eq:appA-primal}
\WEVaR_\varepsilon(X)=\max_{\gamma\ge0}\ \sum_{s,s'}\gamma(s,s')\,X(s')
\quad\text{s.t.}\quad
\sum_{s'}\gamma(s,s')=P(s)\ \ \forall s,\qquad
\sum_{s,s'}d(s,s')\,\gamma(s,s')\le\varepsilon ,
\end{equation}
where the perturbed law is the second marginal $Q(s')=\sum_s\gamma(s,s')$ (automatically in $\simplex(\Sset)$ because $\sum_{s,s'}\gamma=\sum_sP(s)=1$), and the objective equals $\E_Q[X]$. The program is feasible (take $\gamma(s,s')=P(s)\mathbf 1[s'=s]$, of transport cost $0\le\varepsilon$) and bounded ($X$ is bounded), so LP strong duality applies. Attaching a free multiplier $u(s)$ to each marginal equality and $\lambda\ge0$ to the budget, the Lagrangian is
\[
\sum_s P(s)u(s)+\lambda\varepsilon+\sum_{s,s'}\gamma(s,s')\big(X(s')-u(s)-\lambda\,d(s,s')\big).
\]
Its supremum over $\gamma\ge0$ is finite iff $X(s')-u(s)-\lambda\,d(s,s')\le0$ for all $s,s'$, i.e. $u(s)\ge\max_{s'}\big(X(s')-\lambda\,d(s,s')\big)=X^{c}_\lambda(s)$, in which case the supremum equals $\sum_sP(s)u(s)+\lambda\varepsilon$ (attained at $\gamma=0$). Minimizing over feasible $u$ sets $u(s)=X^{c}_\lambda(s)$, giving
\[
\WEVaR_\varepsilon(X)=\min_{\lambda\ge0}\Big\{\lambda\varepsilon+\textstyle\sum_sP(s)\,X^{c}_\lambda(s)\Big\}=\inf_{\lambda\ge0}\big\{\lambda\varepsilon+\E_P[X^{c}_\lambda]\big\},
\]
which is \eqref{eq:wevar-dual}. (For general Polish spaces this is Kantorovich--Rubinstein/Wasserstein-DRO strong duality, \citealp{villani2009,gao2023,blanchet2019}.)

\paragraph{(b) Convexity and $\lambda^\star\le\Lip_d(X)$.}
For each fixed $s$, $\lambda\mapsto X^{c}_\lambda(s)=\max_{s'}\big(X(s')-\lambda\,d(s,s')\big)$ is a pointwise maximum of functions affine in $\lambda$, hence convex; therefore $\E_P[X^{c}_\lambda]=\sum_sP(s)X^{c}_\lambda(s)$ is convex (nonnegative combination), and $\phi(\lambda):=\lambda\varepsilon+\E_P[X^{c}_\lambda]$ is convex on $[0,\infty)$, so \eqref{eq:wevar-dual} is a well-posed one-dimensional convex program. Next, if $\lambda\ge\Lip_d(X)$ then for every $s,s'$, $X(s')-X(s)\le\Lip_d(X)\,d(s,s')\le\lambda\,d(s,s')$, so $X(s')-\lambda\,d(s,s')\le X(s)$ with equality at $s'=s$; hence $X^{c}_\lambda(s)=X(s)$ and $\phi(\lambda)=\lambda\varepsilon+\E_P[X]$, which is strictly increasing in $\lambda$. A convex $\phi$ that is increasing on $[\Lip_d(X),\infty)$ attains its minimum at some $\lambda^\star\le\Lip_d(X)$.

\paragraph{(c) Monotonicity and concavity in $\varepsilon$.}
The feasible set $\{Q:\Wass(Q,P)\le\varepsilon\}$ grows with $\varepsilon$, so $\WEVaR_\varepsilon(X)$ is nondecreasing. By \eqref{eq:wevar-dual} it is an infimum over $\lambda$ of the maps $\varepsilon\mapsto\lambda\varepsilon+\E_P[X^c_\lambda]$, each affine in $\varepsilon$; an infimum of affine functions is concave.

\paragraph{(d) Closed form and saturation.}
Let $T(s)\in\arg\max_{s'}\big(X(s')-\Lip_d(X)\,d(s,s')\big)$ be a steepest-ascent target and set
\(
\bar\varepsilon(X):=\E_P\big[d(s,T(s))\big]=\sum_sP(s)\,d\big(s,T(s)\big).
\)
The right derivative of the convex function $\E_P[X^c_\lambda]$ at $\lambda=\Lip_d(X)$ equals $-\sum_sP(s)\,d(s,T(s))=-\bar\varepsilon(X)$ (each active term contributes the negative source--target distance, inactive terms contribute $0$). Hence the left derivative of $\phi$ at $\Lip_d(X)$ is $\varepsilon-\bar\varepsilon(X)$, which is $\le0$ iff $\varepsilon\le\bar\varepsilon(X)$; in that regime the minimizer is $\lambda^\star=\Lip_d(X)$ and
\[
\WEVaR_\varepsilon(X)=\Lip_d(X)\,\varepsilon+\E_P[X^{c}_{\Lip_d(X)}]=\E_P[X]+\varepsilon\,\Lip_d(X).
\]
For $|\mathrm{supp}(P)\cup\{T(s)\}|=2$ there is a single transport channel, $\bar\varepsilon(X)=\Wass(\delta_{\arg\max X},\,\cdot)$ exhausts only at a vertex, and the formula is exact up to that point. \hfill$\square$

\subsection*{B.\;\;Proposition~\ref{prop:coherent} (Coherence)}
Write $\rho(X)=\WEVaR_\varepsilon(X)=\sup_{Q\in B}\E_Q[X]$ with $B=\{Q:\Wass(Q,P)\le\varepsilon\}$, a nonempty ($P\in B$), convex, compact subset of $\simplex(\Sset)$. \emph{Monotonicity:} if $X\le Y$ pointwise then $\E_Q[X]\le\E_Q[Y]$ for every $Q\in B$ (as $Q\ge0$), so $\rho(X)\le\rho(Y)$. \emph{Translation invariance:} for $c\in\R$, $\E_Q[X+c]=\E_Q[X]+c$ for all $Q$, hence $\rho(X+c)=\rho(X)+c$. \emph{Positive homogeneity:} for $t\ge0$, $\rho(tX)=\sup_{Q\in B}t\,\E_Q[X]=t\,\rho(X)$. \emph{Subadditivity:} $\rho(X+Y)=\sup_{Q\in B}\big(\E_Q[X]+\E_Q[Y]\big)\le\sup_{Q\in B}\E_Q[X]+\sup_{Q\in B}\E_Q[Y]=\rho(X)+\rho(Y)$. These are the axioms of \citet{artzner1999}; $B$ is the risk envelope of the induced dual representation. \hfill$\square$

\subsection*{C.\;\;Theorem~\ref{thm:hier}}

\paragraph{(i) Sweep from mean to worst case.}
$\Wass$ is a metric on $\simplex(\Sset)$, so $\Wass(Q,P)=0\iff Q=P$ and $\WEVaR_0(X)=\E_P[X]$. For any $Q$, $\E_Q[X]\le\max_sX(s)$, so $\WEVaR_\varepsilon(X)\le\max_sX(s)$. Let $s^\bullet\in\arg\max_sX(s)$. Then $\Wass(\delta_{s^\bullet},P)=\sum_sP(s)\,d(s,s^\bullet)\le D$, so $\delta_{s^\bullet}\in B$ once $\varepsilon\ge\sum_sP(s)d(s,s^\bullet)$, whence $\WEVaR_\varepsilon(X)=\max_sX(s)$. Combined with monotonicity (Thm.~\ref{thm:dual}(c)), $\WEVaR_\varepsilon(X)\uparrow\max_sX(s)$ as $\varepsilon\uparrow D$.

\paragraph{Lemma (transport vs.\ total variation).}
For all $Q,P$, $\Wass(Q,P)\le D\cdot\mathrm{TV}(Q,P)$, where $\mathrm{TV}(Q,P)=\tfrac12\sum_s|Q(s)-P(s)|$. \emph{Proof:} let $m(s)=\min(Q(s),P(s))$; the coupling that keeps mass $m$ in place ($\gamma(s,s)\!\ge\! m(s)$) and transports the residual mass $\sum_s(Q(s)-m(s))=\mathrm{TV}(Q,P)$ arbitrarily incurs cost $\le D\cdot\mathrm{TV}(Q,P)$, and $\Wass$ is the minimal cost. \hfill$\diamond$

\paragraph{(ii) Sandwich.}
Fix $\alpha\in(0,1]$ and put $\rho=-\ln\alpha$. For any $Q$ with $\KL(Q\Vert P)\le\rho$, Pinsker's inequality gives $\mathrm{TV}(Q,P)\le\sqrt{\KL(Q\Vert P)/2}\le\sqrt{\rho/2}$, so by the Lemma $\Wass(Q,P)\le D\sqrt{\rho/2}=D\sqrt{-\tfrac12\ln\alpha}$. Therefore
\(
\{Q:\KL(Q\Vert P)\le-\ln\alpha\}\subseteq\{Q:\Wass(Q,P)\le D\sqrt{-\tfrac12\ln\alpha}\},
\)
and taking the supremum of $\E_Q[X]$ over the larger set,
\[
\EVaR_\alpha(X)=\!\!\sup_{\KL(Q\Vert P)\le-\ln\alpha}\!\!\E_Q[X]\ \le\!\!\sup_{\Wass(Q,P)\le D\sqrt{-\frac12\ln\alpha}}\!\!\E_Q[X]=\WEVaR_{\,D\sqrt{-\frac12\ln\alpha}}(X).
\]

\paragraph{(iii) Catastrophe domination.}
Assume $P(s^\star)=0$ and $X(s^\star)>\E_P[X]$. \emph{Entropic measure is blind.} $\E_P[e^{tX}]=\sum_{s:P(s)>0}P(s)e^{tX(s)}$ omits the term $s^\star$, so it---and hence $\EVaR_\alpha(X)=\inf_{t>0}\tfrac1t(\ln\E_P[e^{tX}]-\ln\alpha)$---does not depend on $X(s^\star)$, for every $\alpha$. (Dually, any $Q$ with $\KL(Q\Vert P)<\infty$ satisfies $Q\ll P$, so $Q(s^\star)=0$ and $\E_Q[X]$ ignores $X(s^\star)$.) \emph{Transport measure reacts.} Let $\delta=\mathrm{dist}_d(s^\star,\mathrm{supp}\,P)=\min_{s:P(s)>0}d(s,s^\star)$ and $s_0$ an achieving state. For $\varepsilon>\delta$ choose $\eta_0=\min\big(P(s_0),\varepsilon/\delta\big)>0$ and the feasible law $Q_{\eta_0}=P-\eta_0\delta_{s_0}+\eta_0\delta_{s^\star}$, for which $\Wass(Q_{\eta_0},P)\le\eta_0\,\delta\le\varepsilon$. Then
\[
\WEVaR_\varepsilon(X)\ \ge\ \E_{Q_{\eta_0}}[X]=\E_P[X]+\eta_0\big(X(s^\star)-X(s_0)\big),
\]
which is strictly increasing in $X(s^\star)$. Moreover $\WEVaR_\varepsilon(X)=\max_{Q\in B}\sum_sQ(s)X(s)$ is convex in the vector $X$, and by Danskin's theorem its partial right-derivative in $X(s^\star)$ equals $\max\{Q(s^\star):Q\text{ optimal}\}$; for $X(s^\star)$ large enough any optimal $Q$ must place positive mass on $s^\star$ (else $Q_{\eta_0}$ strictly improves), so the derivative is positive and $\WEVaR_\varepsilon(X)$ strictly increases in $X(s^\star)$ whenever $\varepsilon>\delta$. \hfill$\square$

\subsection*{D.\;\;Theorem~\ref{thm:bellman}}
Fix $(s,a,b)$ and write the continuation vector $W(s')=V\big(s',\psi(b,s,a,s')\big)$ and radius $\varepsilon(b)=\beta\Ent(b)$, nominal $\Pbar_b$.

\paragraph{(i) Closed form of the inner minimization.}
Applying Theorem~\ref{thm:dual} to the loss $-W$ and using $(-W)^{c}_\lambda(s)=\max_{s'}(-W(s')-\lambda d(s',s))=-\min_{s'}(W(s')+\lambda d(s',s))$,
\[
\inf_{Q\in\Uset(b)}\E_Q[W]=-\WEVaR_{\varepsilon(b)}(-W)=\sup_{\lambda\ge0}\Big\{\E_{\Pbar_b}\big[W^{c,-}_\lambda\big]-\lambda\,\varepsilon(b)\Big\},\quad W^{c,-}_\lambda(s)=\min_{s'}\big(W(s')+\lambda\,d(s,s')\big),
\]
a one-dimensional concave maximization, with first-order/unsaturated value $\E_{\Pbar_b}[W]-\beta\Ent(b)\,\Lip_d(W)$ (Thm.~\ref{thm:dual}(d) applied to $-W$). Evaluating $\Lip_d(W)$ costs $O(n^2)$ and the dual is a scalar program, versus an $O(n^2)$-variable transportation LP \eqref{eq:appA-primal}.

\paragraph{(ii) Contraction.}
We use two properties of $\mathfrak T$. \emph{Monotonicity:} if $V\le V'$ then $W\le W'$ pointwise, so $\inf_{Q}\E_Q[W]\le\inf_Q\E_Q[W']$ and, taking $\max_a$, $\mathfrak TV\le\mathfrak TV'$. \emph{Scaled constant shift:} for $c\in\R$, adding $c$ to $V$ adds $c$ to every continuation, and $\inf_Q\E_Q[W+c]=\inf_Q\E_Q[W]+c$, so $\mathfrak T(V+c\mathbf 1)=\mathfrak TV+\gamma c\mathbf 1$. For $\gamma<1$, monotonicity and the scaled shift are exactly the hypotheses of the contraction lemma \citep[Ch.~2]{bertsekas1996}: $\mathfrak TV-\mathfrak TV'\le\gamma\|V-V'\|_\infty\mathbf 1$ and symmetrically, so $\|\mathfrak TV-\mathfrak TV'\|_\infty\le\gamma\|V-V'\|_\infty$, a $\gamma$-contraction. For $\gamma=1$ the shift is exact and gives nonexpansiveness; under properness---every policy together with every kernel selection in the sets $\Uset(\cdot)$ reaches the terminal set $\mathcal T$ in uniformly bounded expected time---the standard stochastic-shortest-path argument \citep[Ch.~3]{bertsekas1996} upgrades this to a contraction in a weighted supremum norm $\|\cdot\|_w$: there exist $m\ge1,\ \kappa\in[0,1)$ with $\|\mathfrak T^mV-\mathfrak T^mV'\|_w\le\kappa\|V-V'\|_w$. Either way $\mathfrak T$ has a unique fixed point $V^\star$.

\paragraph{(iii) Safety sandwich.}
\emph{Upper bound.} Since $\Pbar_b\in\Uset(b)$, $\inf_{Q\in\Uset(b)}\E_Q[W]\le\E_{\Pbar_b}[W]$, so $\mathfrak TV\le\mathfrak T_{\mathrm{Bayes}}V$ pointwise, where $\mathfrak T_{\mathrm{Bayes}}$ is the non-robust operator with kernel $\Pbar_b=\sum_zb(z)P(\cdot\mid s,a,z)$; by monotone iteration $V^\star\le V_{\mathrm{Bayes}}$. Committing to one policy under belief $b$ cannot beat knowing the type, so $V_{\mathrm{Bayes}}(s,b)\le\E_{z\sim b}[V^\star_{\mathrm{opt}}(s,z)]$ (nonnegative value of information), giving the upper bound. \emph{Lower bound.} Let $V^\star_{\mathrm{wc}}$ be the fixed point of the operator $\mathfrak T_{\mathrm{wc}}$ obtained by replacing the radius $\beta\Ent(b)$ with its maximum $\beta\ln|\Zset|$ at \emph{every} $(s,b)$, keeping the same nominal $\Pbar_b$. Since $\beta\Ent(b)\le\beta\ln|\Zset|$, the ambiguity ball in $\mathfrak T$ is contained in that of $\mathfrak T_{\mathrm{wc}}$, so its inner infimum is no smaller; hence $\mathfrak TV\ge\mathfrak T_{\mathrm{wc}}V$ pointwise and, by monotone iteration, $V^\star(s,b)\ge V^\star_{\mathrm{wc}}(s,b)$. This is the bound used in Fig.~\ref{fig:sandwich} and requires no nested-set condition. \emph{Convergence.} As $b\to\delta_{z^\star}$, $\Ent(b)\to0$ and $\varepsilon(b)\to0$, so $\Uset(b)\to\{P(\cdot\mid s,a,z^\star)\}$; by the Lipschitz dependence of the value on the radius (Thm.~\ref{thm:dual}(c): $0\le\WEVaR_\varepsilon-\WEVaR_0\le\varepsilon\Lip_d$), $V^\star(s,b)\to V^\star_{\mathrm{opt}}(s,z^\star)$. \hfill$\square$

\subsection*{E.\;\;Supplementary Experiments and Figures}
All figures are reproduced by the released code (a Colab-ready package). The body shows the geometric core (Figs.~\ref{fig:klw1},~\ref{fig:balls}) and the certified guarantee (Fig.~\ref{fig:sandwich}); the remaining plots are collected here.

\paragraph{Quantitative verification of the dualities.}
Figure~\ref{fig:dro} is the solver-side companion to \S\ref{sec:verify}: across a sweep of radii the entropic and Wasserstein measures both rise from the mean to the worst case, with $\EVaR_\alpha\le\WEVaR$ at the Pinsker-matched radius (Thm.~\ref{thm:hier}(ii)); and a zero-nominal-probability disaster leaves the entropic measure flat while $\WEVaR$ reacts (Thm.~\ref{thm:hier}(iii)).

\begin{figure}[ht]
  \centering
  \includegraphics[width=0.82\textwidth]{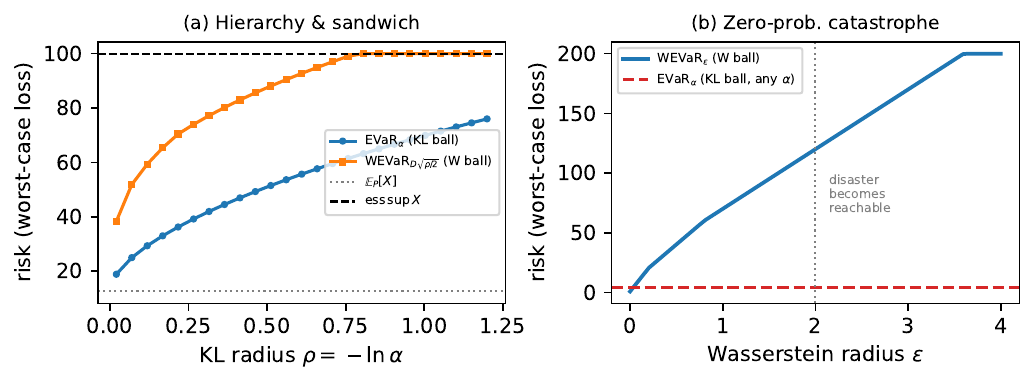}
  \caption{Gurobi-verified comparison on a five-state metric space. \textbf{(a)} hierarchy and the Pinsker sandwich; \textbf{(b)} the zero-probability catastrophe: the entropic value-at-risk is invariant to the disaster's magnitude, $\WEVaR$ is not.}\label{fig:dro}
\end{figure}

\paragraph{Manifolds over the belief simplex.}
For a three-type environment the belief lives on a $2$-simplex, and the controller reads three scalar fields off it (Fig.~\ref{fig:manifold}): the radius $\varepsilon(b)=\beta\Ent(b)$ (an entropy bowl, maximal at the centroid and zero at the vertices), the robust value $V^\star(s_0,b)$, and the policy regions.

\begin{figure}[ht]
  \centering
  \includegraphics[width=\textwidth]{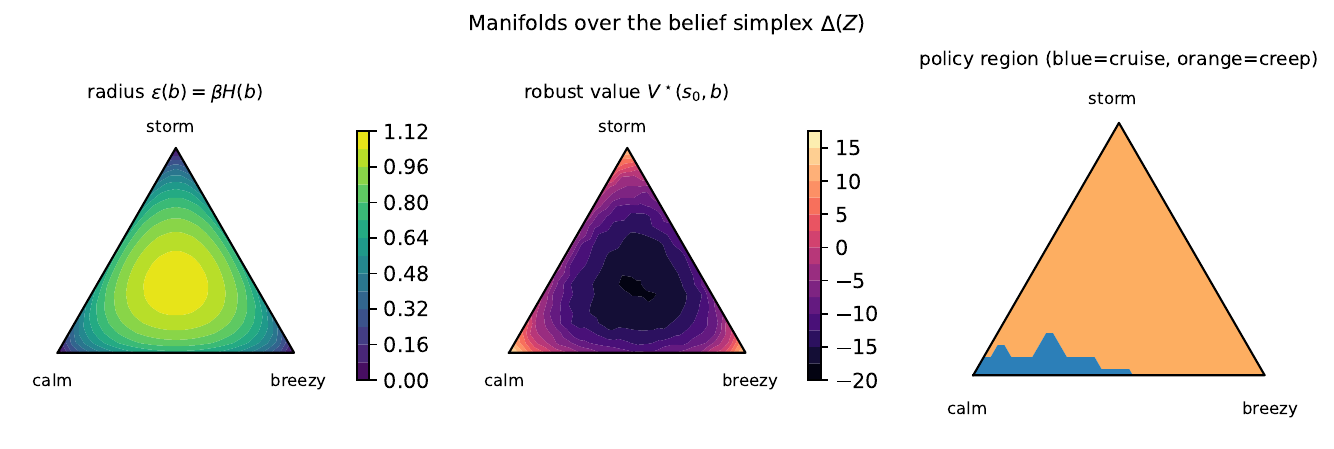}
  \caption{Manifolds over the belief simplex $\simplex(\Zset)$, $|\Zset|=3$ (calm/breezy/storm): ambiguity radius, robust value, and policy region (cruise only near the calm vertex).}\label{fig:manifold}
\end{figure}

\paragraph{Safety switches.}
The belief-dependent policy flips from the safe to the optimal action at a computable threshold (Fig.~\ref{fig:switch}): on the Ambiguous Bridge at $\alpha^\star\approx0.983$ (set by the reward asymmetry, independent of $\beta$), and on the stormy-drone canyon at $b(\mathrm{calm})\approx0.63$.

\begin{figure}[ht]
  \centering
  \begin{minipage}[t]{0.49\textwidth}\centering
    \includegraphics[width=\textwidth]{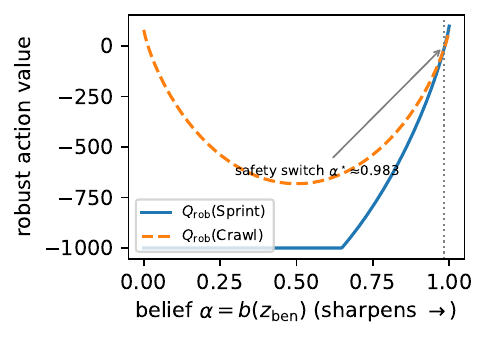}
  \end{minipage}\hfill
  \begin{minipage}[t]{0.49\textwidth}\centering
    \includegraphics[width=\textwidth]{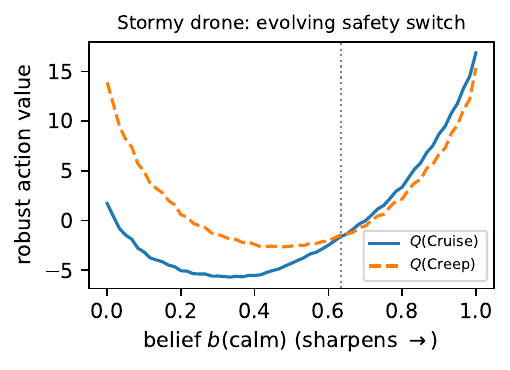}
  \end{minipage}
  \caption{Robust action values vs.\ belief. Left: the Ambiguous Bridge (\textsc{Sprint} vs.\ \textsc{Crawl}, switch at $\alpha^\star\approx0.983$). Right: the stormy-drone canyon (\textsc{Cruise} vs.\ \textsc{Creep}, switch at $b(\mathrm{calm})\approx0.63$).}\label{fig:switch}
\end{figure}

\paragraph{Safety-vs-efficiency rollouts, and an honest caveat.}
Figure~\ref{fig:pareto} compares the three planners on the canyon. $\WEVaR$ Pareto-dominates the static-robust (\textsc{Maximin}) planner---comparable safety at higher return, the ``un-freezing'' effect. In this \emph{calibrated-belief} setting the expected-value (Bayesian) planner is already near-optimally cautious under a severe catastrophe, so the empirical separation between $\WEVaR$ and Bayesian is small; the regime in which belief-scaled caution measurably reduces failures is one of \emph{miscalibrated} belief or distribution shift, where the belief is learned rather than computed from a known model. The body's contribution is therefore the coherent measure, its geometry, and the certified guarantee; this rollout is illustrative.

\begin{figure}[ht]
  \centering
  \includegraphics[width=0.5\textwidth]{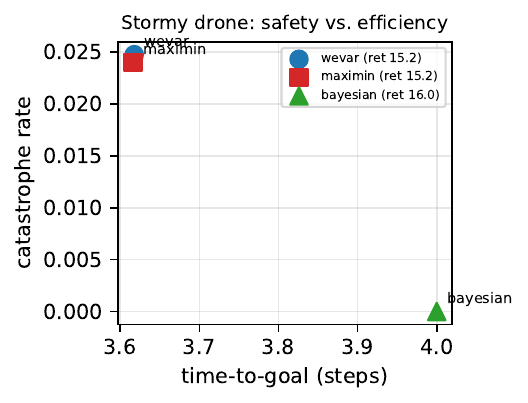}
  \caption{Catastrophe rate vs.\ time-to-goal on the stormy-drone canyon ($4000$ episodes). $\WEVaR$ dominates the static-robust \textsc{Maximin} baseline; the Bayesian planner is already cautious here (see caveat).}\label{fig:pareto}
\end{figure}

\end{document}